\documentclass[sigconf]{acmart}
\usepackage{balance}
\usepackage{amsmath, amsfonts}
\usepackage{algorithmic}
\usepackage{algorithm}
\usepackage{array}
\usepackage[caption=false, font=normalsize, labelfont=sf, textfont=sf]{subfig}
\usepackage{textcomp}
\usepackage{stfloats}
\usepackage{enumitem}
\usepackage{url}
\usepackage{verbatim}
\usepackage{graphicx}
\usepackage{hyperref}
\usepackage{cleveref}
\usepackage{multirow}
\usepackage{amsmath}
\usepackage{tikz}
\usetikzlibrary{fit, positioning}
\usepackage{booktabs}
\usetikzlibrary{fit}
\usepackage{amsthm}

\newtheorem{proposition}{Proposition}

\usepackage{caption}
\acmSubmissionID{1927}
\AtBeginDocument{%
  }

\setcopyright{acmlicensed}
\copyrightyear{2018}
\acmYear{2018}
\acmDOI{XXXXXXX.XXXXXXX}
\acmConference[Conference acronym 'XX]{Make sure to enter the correct
  conference title from your rights confirmation email}{June 03--05,
  2018}{Woodstock, NY}

\acmISBN{978-1-4503-XXXX-X/2018/06}

\copyrightyear{2025}
\acmYear{2025}
\setcopyright{acmlicensed}\acmConference[MM '25]{Proceedings of the 33rd ACM International Conference on Multimedia}{October 27--31, 2025}{Dublin, Ireland}
\acmBooktitle{Proceedings of the 33rd ACM International Conference on Multimedia (MM '25), October 27--31, 2025, Dublin, Ireland}
\acmDOI{10.1145/3746027.3755070}
\acmISBN{979-8-4007-2035-2/2025/10}

\begin{document}

\title{Divide-Then-Rule: A Cluster-Driven Hierarchical Interpolator for Attribute-Missing Graphs}


\author{Yaowen Hu}
\orcid{0000-0003-2081-9358}
\affiliation{
    \department{$^1$College of Computer Science and Technology}
   \institution{$^1$National University of Defense Technology}
   \institution{$^2$Xiangjiang Laboratory}
   \city{Changsha}
   \country{China}
}
\email{yaowenhu@nudt.edu.cn}

\author{Wenxuan Tu}
\orcid{0000-0002-1353-2968}
\authornote{Corresponding authors.}
\affiliation{
  \department{School of Computer Science and Technology}
  \institution{Hainan University}
  \city{Haikou}
  \country{China}}
\email{wenxuantu@163.com}

\author{Yue Liu}
\orcid{0000-0002-9894-0062}
\affiliation{
    \department{College of Computer Science and Technology}
   \institution{National University of Defense Technology}
   \city{Changsha}
   \country{China}
}

\author{Miaomiao Li}
\orcid{0000-0001-7678-687X}
\affiliation{
   \institution{Changsha University}
   \city{Changsha}
   \country{China}
}

\author{Wenpeng Lu}
\orcid{0000-0002-1840-3540}
\affiliation{
    \department{Key Laboratory of Computing Power Network and Information Security}
   \institution{Qilu University of Technology}
   \city{Jinan}
   \country{China}
}

\author{Zhigang Luo}
\orcid{0000-0002-7552-201X}
\affiliation{
   \institution{Xiangjiang Laboratory}
   \city{Changsha}
   \country{China}
}

\author{Xinwang Liu}
\authornotemark[1]
\orcid{0000-0001-9066-1475}
\affiliation{
\department{College of Computer Science and Technology}
   \institution{National University of Defense Technology}
   \city{Changsha}
   \country{China}
}
\email{xinwangliu@nudt.edu.cn}

\author{Ping Chen}
\orcid{0000-0003-0684-9665}
\affiliation{
\department{School of Information and Communication Engineering}
   \institution{North University of China}
   \city{Taiyuan}
   \country{China}
}

\begin{abstract}
Deep graph clustering (DGC) for attribute-missing graphs is an unsupervised task aimed at partitioning nodes with incomplete attributes into distinct clusters. Existing imputation methods for attribute-missing graphs often fail to account for the varying amounts of information available across node neighborhoods, leading to unreliable results. To address this issue, we propose a novel method named Divide-Then-Rule Graph Completion (DTRGC). This method first addresses nodes with sufficient known neighborhood information and treats the imputed results as new knowledge to iteratively impute more challenging nodes, while leveraging clustering information to correct imputation errors. Specifically, Dynamic Cluster-Aware Feature Propagation initializes missing node attributes by adjusting propagation weights based on the clustering structure. Subsequently, Hierarchical Neighborhood-Aware Imputation categorizes attribute-missing nodes into three groups based on the completeness of their neighborhood attributes. The imputation is performed hierarchically, prioritizing the groups with nodes that have the most available neighborhood information. The cluster structure is then used to refine the imputation and correct potential errors. Finally, Hop-wise Representation Enhancement integrates information across multiple hops, thereby enriching the expressiveness of node representations. Experimental results on 6 widely used graph datasets show that DTRGC significantly improves the clustering performance of various DGC methods under attribute-missing graphs.
\end{abstract}


\maketitle

\section{Introduction}
Deep Graph Clustering (DGC) \cite{liuyue_deep_graph_clustering_survey,9944925,yin2024continuous,ju2024survey,yin2024dynamic} is a fundamental task in unsupervised learning that aims to cluster nodes based on graph structures, grouping similar nodes together without reliance on labels. 
In recent years, DGC methods have achieved notable success in various practical applications, including community detection, recommendation systems \cite{yin2023dream,liuyue_ITR,Darec,yin2023omg}, and facial recognition \cite{yin2023messages,park2022cgc,pang2023sa}. Existing efforts are typically based on contrastive learning \cite{yoo2022accurate,tu2022initializing} or reconstruction models \cite{yin2023coco,li2023csat}. These methods either maximize the similarity between nodes in the same cluster or reconstruct node features to enhance clustering effectiveness \cite{Graphlearner,yin2022deal}.

Existing DGC methods operate under a crucial assumption that the attributes of all nodes are complete and reliable \cite{yin2022dynamic, yang2023cluster,ai2023gcn,yu2024dvsai,Wan24Fast,yin2022generic}. However, in real-world scenarios, issues like privacy protection and copyright restrictions during data collection often lead to incomplete node attributes. Attribute-missing graphs typically manifest in two patterns: (1) Uniform missing, where some attributes of each node are partially missing \cite{huo2023t2,shou2023adversarial}; and (2) Structural missing, where all attributes of certain nodes are entirely missing \cite{tang2024merging,yinsport}. Although considerable attention has been devoted to imputing attribute-missing graphs \cite{zhang2022would, tu2024attribute}, most of these studies encounter significant challenges in clustering tasks: (1) Semi-supervised imputation methods rely heavily on partial label information and cannot be directly applied to clustering. For example, label propagation (LP) \cite{zhang2023investigating} propagates known labels to unlabeled nodes via neighboring nodes, making it unsuitable for fully unsupervised clustering tasks. (2) Plug-and-play imputation methods often lack a mechanism to account for clustering structures. For instance, feature propagation (FP) \cite{meng2024deep} propagates known features across edges between nodes without regard to cluster membership, maintaining the same message-passing (MP) intensity regardless of whether two nodes belong to the same cluster. 

\begin{figure*}[htbp]
\centering
\includegraphics[width=0.62\linewidth]{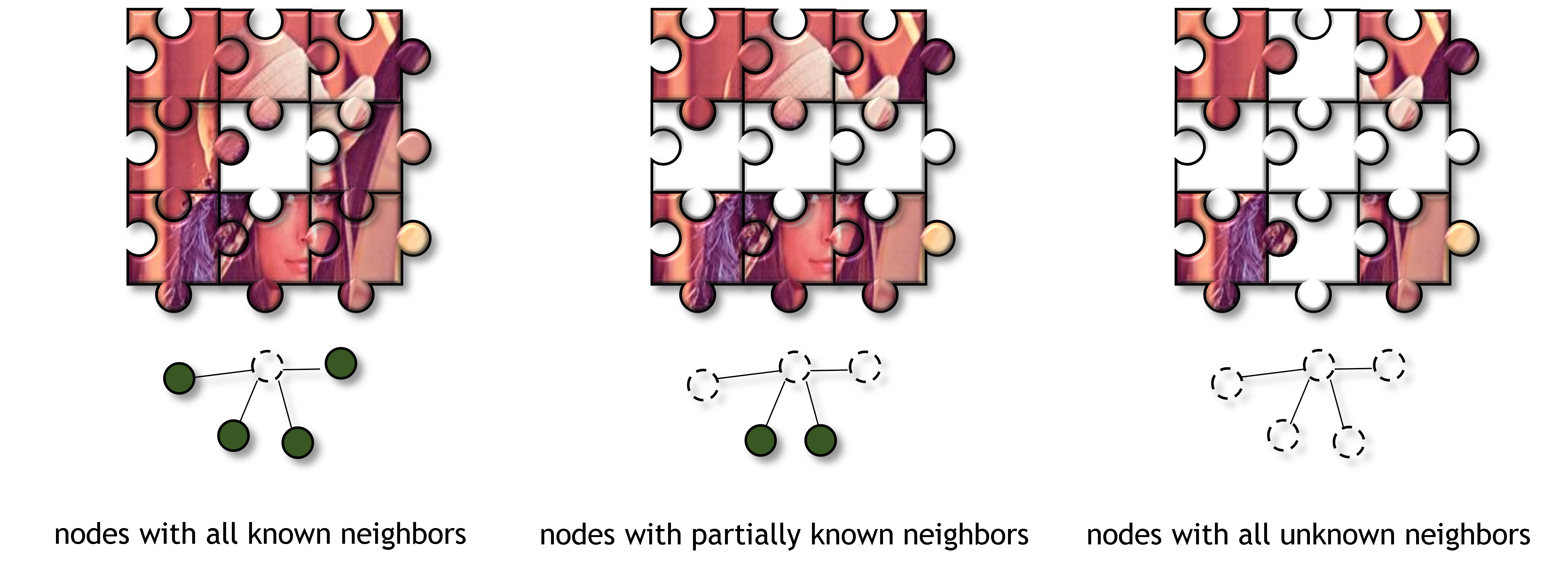}
\caption{Illustration of the hierarchical differentiation of nodes based on the completeness of their neighborhood attributes. The figure employs different jigsaw puzzle states to represent three categories of nodes: ``nodes with all known neighbors'' are akin to well-defined edge pieces of a jigsaw puzzle; ``nodes with partially known neighbors'' correspond to jigsaw puzzle regions that are partially assembled; and ``nodes with all unknown neighbors'' resemble blank sections lacking sufficient contextual information. }
\label{figure1}
\end{figure*}

Despite the fact that the aforementioned methods improve clustering performance on attribute-missing graphs through imputation, they share a critical limitation: \textit{These methods fail to account for differences in neighborhood information and therefore produce unreliable results for nodes with insufficient known information.} For example, when ``nodes with all unknown neighbors'' are interconnected, the imputation results for the internal missing nodes become unreliable. To overcome these limitations, we draw inspiration from the principle of Prägnanz in Gestalt psychology \cite{van2024pragnanz}. Specifically, when humans encounter an incomplete scene, they tend to prioritize known and high-confidence information, using it to progressively infer the unknown. As illustrated in Fig.\ref{figure1}, an attribute-missing graph can be analogized to a partially disassembled jigsaw puzzle: while the adjacency relations between pieces (i.e., the graph structure) remain intact, the patterns on some pieces (i.e., node attributes) are missing.

Based on these observations, there is an urgent need for a graph imputation scheme that supports hierarchical processing and is tailored to different types of attribute-missing nodes. We propose a novel approach, termed \textbf{D}ivide-\textbf{T}hen-\textbf{R}ule \textbf{G}raph \textbf{C}ompletion (DTRGC), which iteratively imputes nodes by first processing those with sufficient known neighborhood information, treating the imputed results as new knowledge, and leveraging clustering information to correct imputation errors. DTRGC operates in three stages: First, Dynamic Cluster-aware Feature Propagation (DCFP) initializes missing node attributes and adjusts propagation weights based on the clustering structure. This enhances feature transmission among intra-cluster nodes and reduces interference from heterogeneous nodes. Next, Hierarchical Neighborhood-Aware Imputation (HNAI) categorizes nodes into three types based on attribute completeness within their neighborhoods. It applies distinct imputation strategies to refine the imputation process. Finally, Hop-wise Representation Enhancement (HRE) enriches node representations by gradually expanding the neighborhood scope and integrating multi-hop information. 

The main contributions of our work are as follows:

\begin{itemize} 
\item This paper presents DTRGC, the first framework tailored for clustering on attribute-missing graphs via hierarchical processing, in which nodes are classified based on the completeness of their neighborhood attributes and are imputed using category-specific strategies.

\item Three core components are developed within DTRGC: DCFP, HNAI and HRE. These modules operate synergistically, progressively imputing missing node attributes and ensuring that each node category receives appropriate treatment, thus enhancing overall imputation quality.

\item Experimental evaluations on six widely used graph datasets show that DTRGC improves the clustering performance of all tested DGC methods on the attribute-missing graphs. 

\end{itemize}

\section{Related Work}

\subsection{Attribute-Complete Deep Graph Clustering}
In recent years, DGC has made significant advancements in handling attribute-complete graphs. Existing methods are broadly classified into two categories: (1) Graph Autoencoder-based Methods. These methods utilize decoders to reconstruct the original features from the learned node embeddings \cite{tao2019adversarial, pan2019learning,tu2021deep,tu2022initializing, cai2021learning}. For instance, the Embedding Graph AutoEncoder (EGAE) integrates $k$-means clustering with autoencoder training \cite{zhang2022embedding}. This integration enhances the consistency between node representations and clustering. (2) Contrastive Learning-based Methods. These methods optimize node embeddings by creating positive and negative sample pairs, aiming to pull positive samples closer while pushing negative samples apart in the embedding space \cite{liuyue_DCRN,liuyue_HSAN,liuyue_Dink_net,liuyue_SCGC,liuyue_IDCRN}. The Simple Contrastive Graph Clustering (SCGC) improves clustering efficiency through simplified data augmentation and a streamlined architecture \cite{liuyue_SCGC}. Similarly, Peng et al. \cite{peng2023dual} introduced the Dual Contrastive Learning Network (DCLN), which incorporates dual mechanisms to prevent feature collapse and enhance the discriminative power of node embeddings. 

\subsection{Attribute-Missing Graph Imputation}
The objective of attribute-missing graph imputation is to enable effective learning when node attributes are partially or entirely absent. Existing approaches can be broadly categorized into two classes. (1) GCN-variant Models, directly operate on attribute-incomplete graphs and recover missing features during training. These methods typically extend Graph Convolutional Networks (GCNs) \cite{kipf2016semi,10158394,10634240,9272360} to handle incomplete attributes. For example, the Structural Attribute Transformer (SAT) \cite{chen2020learning} models the joint distribution of graph structure and node features via distributional learning. GCNMF \cite{taguchi2021graph} integrates a Gaussian mixture model to address missing features, while PaGNN \cite{jiang2020incomplete} proposes a partial MP scheme using only available attributes. 
(2) Feature Imputation Methods, reconstructs missing attributes prior to learning, without relying on parameterized encoders or training \cite{jin2022amer, yoo2022accurate}.  A classical method is label propagation (LP) \cite{xiaojin2002learning}, which infers unknown labels by diffusing known ones, but it requires labels and is unsuitable for clustering. To overcome this limitation, recent methods such as Feature Propagation (FP) \cite{rossi2022unreasonable} and Pseudo-Confidence Feature Imputation (PCFI) \cite{um2023confidence} leverage graph topology to recover missing features. FP performs feature diffusion, while PCFI assigns confidence scores to each imputed feature.

\section{Proposed Method}

\begin{figure*}[htbp]
\centering
\includegraphics[width=1\linewidth]{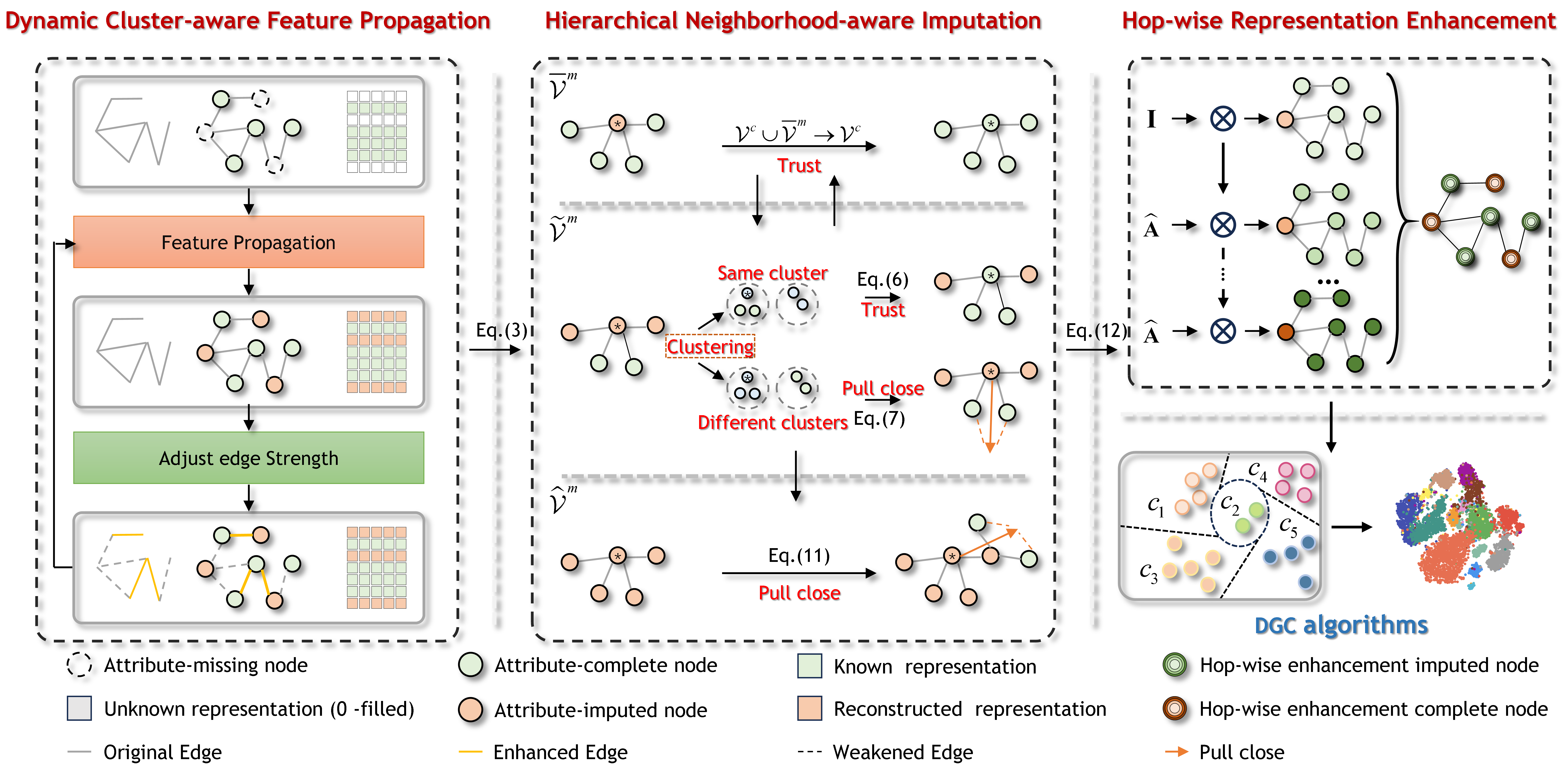}
\caption{The overall architecture of the DTRGC. In this framework, we first initialize the missing node attributes while adjusting propagation weights based on cluster structures to enhance intra-cluster propagation and suppress inter-cluster propagation. Next, the attribute-missing nodes are categorized into three groups based on the completeness of their neighborhood attributes, and distinct imputation strategies are applied to refine the imputation process for each group. Finally, multi-hop propagation is utilized to extend the neighborhood scope.}
\label{figure2}
\end{figure*}

\subsection{Notations and Problem Definition}

\textbf{Notations.} \textit{Let $\mathcal{G} = \{\mathcal{V},\mathcal{E}, \mathbf{H}\}$ denote an undirected graph, where $\mathcal{V}=\{v_{n}\}_{n=1}^{N}$ is the vertex set with $N$ nodes, $\mathcal{E}$ is the edge set, $\mathbf{H} \in\mathbb{R}^{N\times d}$ is the node attribute matrix where $d$ denotes the feature dimension of node. $\mathbf{A}$ is the adjacency matrix representing the relationships between nodes. Specifically, ${a}_{ij}=1$ if there is an edge between nodes $v_i$ and $v_j$, and $a_{ij}=0$ otherwise.}

\noindent
\textit{\textbf{Definition 1.}} \textbf{ (Attribute-Missing Graph)}
\textit{ The attribute-missing graph refers to a graph in which certain nodes lack attributes entirely, meaning that some nodes have no associated feature representations. Thus, the nodes set $\mathcal{V}$ of the graph can be divided into two subsets: the attribute-complete nodes $\mathcal{V}^{c}=\{v_{n}^{c}\}_{n=1}^{N^{c}}$ and the attribute-missing nodes $\mathcal{V}^{m}=\{v_{n}^{m}\}_{n=1}^{N^{m}}$. Accordingly, $\mathcal{V}=\mathcal{V}^{c}\cup \mathcal{V}^{m},\mathcal{V}^{c}\cap \mathcal{V}^{m}=\emptyset $ and $N=N^{c}+N^{m}$.
Based on these notations, an attribute-missing graph can be formulated as $\widetilde{\mathcal{G}}=\{\mathcal{V}^{m},\mathcal{E}\}$.}

\noindent
\textit{\textbf{Definition 2. }}\textbf{(Attribute-Missing Graph Clustering)}\textit{ Given an unlabeled attribute-missing graph $\widetilde{\mathcal{G}}=\{\mathcal{V}^{m},\mathcal{E}\}$ with $K$ clusters, the objective is to impute cluster-compatible features for the attribute-missing nodes.}






\subsection{Divide-Then-Rule Graph Completion}
As previously discussed, existing methods often fail to account for the varying amounts of available information across node neighborhoods, leading to unreliable results, especially for nodes with insufficient known neighborhood. The working mechanisms of DTRGC are detailed in this section, with implementation specifics provided in Algorithm~\ref{alg:DTRGC}.

\subsubsection{Dynamic Cluster-aware Feature Propagation}

DCFP dynamically adjusts the propagation intensity based on whether nodes belong to the same cluster, thereby reinforcing feature propagation within clusters while suppressing feature diffusion across clusters.
In each iteration, node features are propagated according to the normalized adjacency matrix, followed by a backtracking step to ensure that attribute-complete nodes remain unchanged. 
\begin{equation}\label{eq4_simple}
\mathbf{H}^{(t+1)}=\hat{\mathbf{A}}\mathbf{H}^{(t)},\qquad 
\mathbf{H}^{(t)}_c\equiv\mathbf{H}^{(0)}_c,\;\forall\,t\ge0.
\end{equation}

To leverage cluster structural information effectively, DCFP performs $k$-means clustering after every $T$ iteration and subsequently adjusts the adjacency matrix to reflect the updated cluster structure. Specific adjustment as follows:
\begin{equation}
    \label{eq:update_adj}
    \begin{aligned}
    {\hat{{a}}_{ij}^{(t+1)}}= 
    \begin{cases}
    \alpha\hat{{a}}_{ij}^{(t)}, & \text{if } c_i = c_j \text{ and }t\text{ mod }\, T = 0 \\
    \beta\hat{{a}}_{ij}^{(t)}, & \text{if } c_i \neq c_j \text{ and }t\text{ mod }\, T = 0 \\
    \hat{{a}}_{ij}^{(t)}, & \text{otherwise},
    \end{cases}
    \end{aligned}
\end{equation}
where $c_i$ and $c_j$ represent the clusters to which nodes $i$ and $j$ belong, respectively. Setting $\alpha>1$ strengthens intra‑cluster propagation, while $\beta<1$ weakens inter‑cluster propagation.

Rossi et al. \cite{rossi2022unreasonable} show this update converges and specify the needed iterations, so we set $T=40$ in this work. Feature propagation and adjacency updating then alternate: the adjusted matrix $\hat{\mathbf{A}}_{ij}^{(t)}$ is used for the next $T$ iterations.
\subsubsection{Hierarchical Neighborhood-Aware Imputation}

Although DCFP achieves alternating optimization of clustering and imputation, it does not fully utilize the neighborhood structural information for targeted imputation. HNAI remedies this by progressively imputing: it first fills information‑rich nodes and then uses them to guide the sparse ones.

Specifically, HNAI classifies attribute-missing nodes into three categories based on their 1-hop neighborhood: (1) $\hat{\mathcal{V}}^{m}$, where all neighbors are attribute-complete; (2) $\tilde{\mathcal{V}}^{m}$, where some neighbors are attribute-complete and others are attribute-missing; and (3) $\hat{\mathcal{V}}^{m}$, where all neighbors are attribute-missing. These categories are illustrated in Fig.~\ref{figure1}, and the classification process is formally defined by the following equation:
\begin{equation}
\label{eq6}
    \begin{aligned}
    v_i\in\mathcal{V}^{m}=
    \begin{cases}
    \bar{\mathcal{V}}^{m}&\mathrm{if~}\forall v_j\in \mathcal{N}(i), v_j\in \mathcal{V}^{c}\\
   \tilde{\mathcal{V}}^{m}&\mathrm{if~}\exists v_j\in \mathcal{N}(i), v_j\in \mathcal{V}^{c}\\
    \hat{\mathcal{V}}^{m} &\mathrm{if~}\forall v_j\in \mathcal{N}(i), v_j\in \mathcal{V}^{m},
    
    \end{cases}
    \end{aligned}
\end{equation}
where $v_i$ denotes an attribute-missing node, $v_j$ represents a neighboring node of the attribute-missing node, and $\mathcal{N}(i)$ indicates the 1-hop neighborhoods of node $i$. 

Subsequently, specific imputation strategies are applied to each category of attribute-missing nodes. The iterative process involves: (1) imputing $\hat{\mathcal{V}}^{m}$, (2) imputing $\tilde{\mathcal{V}}^{m}$, and (3) reclassification using Eq.~(\ref{eq6}). Finally, $\hat{\mathcal{V}}^{m}$ nodes, the most challenging to impute, are estimated last.

\noindent\textbf{Nodes with All Known Neighbors} $\hat{\mathcal{V}}^{m}$. 
Based on the homogeneity assumption, these neighboring nodes typically belong to the same cluster during the clustering process. Thus, $\hat{\mathcal{V}}^{m}$ can be directly treated as attribute-complete nodes, as justified in \textbf{Proposition 1}:

\begin{equation} \label{eq:7} \mathcal{V}^c \leftarrow \mathcal{V}^c \cup \hat{\mathcal{V}}^{m}. \end{equation}

\begin{proposition}
For any $v_i \in \hat{\mathcal{V}}^{m}$, its feature imputation is based on a non-negative linear combination of its fully observed neighbors in the 1-hop neighborhood. Therefore, it is reasonable to treat it as an attribute-complete node.
\end{proposition}

\begin{proof}
For any $v_i \in \hat{\mathcal{V}}^{m}$, its imputation process entirely depends on the attribute-complete nodes in its 1-hop neighborhood. According to Eq.(\ref{eq4_simple}), the imputed features of a node can be represented as a non-negative linear combination of its attribute-complete neighbors, normalized by the adjacency matrix. Specifically, for any $v_i \in \hat{\mathcal{V}}^{m}$, the imputed feature $\mathbf{x}_i$ can be written as: 
\begin{equation} 
\label{eq11} 
    \mathbf{x}_i = \sum_{v_j \in \mathcal{N}(i)}\hat{a}_{ij} \mathbf{x}_j, 
\end{equation} 
where $\hat{a}_{ij} \in (0,1)$ and $\sum\limits_{v_j \in \mathcal{N}(i)}\hat{a}_{ij} \leq 1$. Thus, this implies that the imputed node feature $\mathbf{x}_i$ lies within the convex hull defined by the features of its neighbors, ensuring that it is located within the feature space formed by these neighbors.

In clustering, distance-based clustering algorithms (e.g., $k$-means) group points that are closer together into the same cluster. As a result, the imputed node $v_i$ correctly reflects its position within the cluster structure and maintains consistency with its neighbors. 
\end{proof}

\noindent\textbf{Nodes with Partially Known Neighbors} $\tilde{\mathcal{V}}^{m}$. The completion of this category of nodes relies on two sources of information: (1) direct propagation from attribute-complete neighbors within their 1-hop neighborhood; and (2) indirect transmission of information from distant attribute-complete nodes through attribute-missing neighbors. To improve imputation accuracy, we apply the following refinement strategy for this category of nodes:

\noindent\textit{Intra-cluster Reinforcement Strategy (IRS).} If a $\tilde{\mathcal{V}}^{m}$ node and all its attribute‑complete neighbors lie in the same cluster, the homophily assumption holds. We therefore deem its imputed feature reliable and promote the node to the attribute‑complete set $\mathcal{V}^c$:
\begin{equation}
    \begin{aligned}
    \label{eq:13}
    &\mathcal{V}^c \leftarrow \mathcal{V}^c \cup \{v_i\}, \quad
    \forall v_i \in \tilde{\mathcal{V}}^{m}
\\
        &\text{if~} \, \forall v_j \in \mathcal{N}_\text{known}(i), c_i = c_j.
    \end{aligned}
\end{equation}

\noindent\textit{Inter-cluster Correction Strategy (ICS).} When a $\tilde{\mathcal{V}}^{m}$ node and all attribute-complete neighbors do not belong to the same cluster, this significantly violates the homogeneity assumption for graphs. To mitigate this effect, we update the feature of this node using EMA \cite{haynes2012exponential}, enabling it to gradually move toward the centroid of the features from attribute-complete neighbors, as described by the following equation:
\begin{equation}
    \begin{aligned}
    \label{eq:12}
   &x_i \leftarrow \gamma x_i + (1 - \gamma) \cdot 
    \mathrm{mean}(\mathcal{N}_\text{known}(i)), \forall v_i \in \tilde{\mathcal{V}}^{m} 
\\
        &\mathrm{if~} \, \forall v_j \in \mathcal{N}_\text{known}(i), c_i \neq c_j,
    \end{aligned}
\end{equation}
where $\gamma$ represents the smoothing coefficient for EMA.

It is important to note that the IRS and ICS may trigger a chain reaction within the neighborhood of certain ${v_i} \in \tilde{\mathcal{V}}^{m}$, causing nodes categorized as $\tilde{\mathcal{V}}^{m}$ and $\hat{\mathcal{V}}^{m}$ within this neighborhood to undergo corresponding state transitions, as demonstrated in \textbf{Proposition 2}.

\begin{proposition}
When a node $v_i \in \tilde{\mathcal{V}}^{m}$ is reclassified as an attribute-complete node (i.e., $v_i \in \mathcal{V}^{c}$):

\begin{enumerate}[label=(\Roman*)]
    \item The neighboring nodes in $\tilde{\mathcal{V}}^{m}$ \textbf{may} be reclassified as $\hat{\mathcal{V}}^{m}$.
    \item The neighboring nodes in $\hat{\mathcal{V}}^{m}$ \textbf{will certainly} be reclassified as $\tilde{\mathcal{V}}^{m}$.
\end{enumerate}
\end{proposition}

\begin{proof}
Consider a node $v_i \in \tilde{\mathcal{V}}^{m}$ to which either the IRS or the ICS has been applied. According to these strategies, if certain conditions are met, $v_i$ will be reclassified as an attribute-complete node and added to the $\mathcal{V}^c$:
\begin{equation}
    \mathcal{V}^c \leftarrow \mathcal{V}^c \cup \{v_i\}.
\end{equation}

\noindent\textbf{Effect on Neighboring $\tilde{\mathcal{V}}^{m}$.} Let $v_j$ be a neighbor of $v_i$, where $v_j \in \tilde{\mathcal{V}}^{m}$. After $v_i$ is updated to an attribute-complete node ($v_i \in \mathcal{V}^c$), we consider the impact on $v_j$:

\begin{itemize}
    \item \textbf{Case 1:} $v_i$ is the only attribute-missing neighbor of $v_j$.\\
In this case, all neighbors of $v_j$ are now attribute-complete nodes:
    \begin{equation}
        \forall v_k \in \mathcal{N}(v_j), \quad v_k \in \mathcal{V}^c.
    \end{equation}

According to Eq.(\ref{eq6}), $v_j$ is then reclassified as $\hat{\mathcal{V}}^{m}$.
    \item \textbf{Case 2:} $v_j$ has other attribute-missing neighbors in addition to $v_i$.\\
Here, $v_j$ will still have both attribute-complete and attribute-missing neighbors, and thus will remain in $\tilde{\mathcal{V}}^{m}$.
\end{itemize}

\noindent\textbf{Effect on Neighboring $\hat{\mathcal{V}}^{m}$.} Consider a neighboring node $v_j \in \mathcal{N}(v_i)$, where initially $v_j \in \hat{\mathcal{V}}^{m}$. After $v_i$ is reclassified as an attribute-complete node (i.e., $v_i \in \mathcal{V}^c$), $v_j$ will gain at least one attribute-complete neighbor, specifically $v_i$:
\begin{equation} 
    \exists v_i \in \mathcal{N}(v_j), \quad v_i \in \mathcal{V}^c.
\end{equation}

In this scenario, $v_j$ no longer meets the definition of being in $\hat{\mathcal{V}}^{m}$, as it now has an attribute-complete neighbor. Therefore, $v_j$ is reclassified as a node in $\tilde{\mathcal{V}}^{m}$.

\end{proof}

\noindent\textbf{Nodes with All Unknown Neighbors} $\hat{\mathcal{V}}^{m}$. As analyzed in \textbf{Proposition 2}, a significant proportion of nodes initially classified as $\hat{\mathcal{V}}^{m}$ are eventually reclassified as $\tilde{\mathcal{V}}^{m}$. This result is also quantitatively reflected in Section \ref{Clustering Performance Comparison (Q1)}. For the remaining $\hat{\mathcal{V}}^{m}$ nodes, their imputation is derived entirely from neighbors that are themselves attribute-missing. Referring to the aforementioned ICS, if a node $v_i \in \hat{\mathcal{V}}^{m}$ has neighbors in $\mathcal{N}_\text{some-known}(i)$ that do not belong to the same cluster, EMA is employed to iteratively adjust the feature of $v_i$, driving it towards the centroid of its $\tilde{\mathcal{V}}^{m}$ neighbors:

\begin{equation}
    \begin{aligned}
    \label{eq:17}
   &x_i \leftarrow \gamma x_i + (1 - \gamma) \cdot 
    \mathrm{mean}(\mathcal{N}_\text{some-known}(i)), \forall v_i \in \bar{\mathcal{V}}^{m} 
\\
        &\mathrm{if~} \, \forall v_j \in \mathcal{N}_\text{some-known}(i), c_i \neq c_j,
    \end{aligned}
\end{equation}
where $\mathrm{mean}(\mathcal{N}_\text{some-known}(i))$ denotes the mean feature vector of all neighboring nodes classified as $\tilde{\mathcal{V}}^{m}$. 

 \begin{algorithm}[t]
    \caption{Divide-Then-Rule Graph Completion (DTRGC)}
    \label{alg:DTRGC}
    \renewcommand{\algorithmicrequire}{\textbf{Input:}} 
    \renewcommand{\algorithmicensure}{\textbf{Output:}} 
    \begin{algorithmic}[1]
        \REQUIRE 
        Feature matrix $\mathbf{H}$; Normalized adjacency matrix $\hat{\mathbf{A}}$; Node sets $\mathcal{V}^{m}$, $\mathcal{V}^{c}$; 
        Maximum epochs ${F_{\text{max}}}$, $I_{\text{max}}$; Maximum propagation step $K$.
        \ENSURE 
        Completed and enhanced feature matrix $\mathbf{H}_E \in \mathbb{R}^{n \times (K+1) d}$
        
\texttt{/* Stage 1: DCFP */}
        \FOR{$t = 1$ to $F_{\text{max}}$}
            \STATE Update $\mathbf{H}^{(t+1)} \leftarrow \hat{\mathbf{A}} \mathbf{H}^{(t)}$
            
            \STATE Restore $\mathbf{H}^{(t+1)}_c \leftarrow \mathbf{H}^{(0)}_c$
            \IF{$t \bmod 40 == 0$}
                \STATE Update $\hat{\mathbf{A}}$ by Eq.(\ref{eq4_simple})
            \ENDIF
        \ENDFOR
        
        \texttt{/* Stage 2: HNAI */} 
        \FOR{$i = 1$ to $I_{\text{max}}^{(H)}$}
            \STATE Obtain $\hat{\mathcal{V}}^{m}$, $\tilde{\mathcal{V}}^{m}$, $\hat{\mathcal{V}}^{m}$ by Eq.(\ref{eq:update_adj})
            \STATE Update $\hat{\mathcal{V}}^{m}$ by Eq.(\ref{eq:7})
            \FOR{each node $v \in \tilde{\mathcal{V}}^{m}$}
                \STATE Update $\tilde{\mathcal{V}}^{m}$ by Eq.(\ref{eq:12})
                \STATE Update $\tilde{\mathcal{V}}^{m}$ and $\mathcal{V}^{m}$ by Eq.(\ref{eq:13})
            \ENDFOR
        \ENDFOR
        
       
        \FOR{each node $v \in \hat{\mathcal{V}}^{m}$}
             \STATE Update $\hat{\mathcal{V}}^{m}$ by Eq.(\ref{eq:17})
        \ENDFOR
         
         \texttt{/* Stage 3: HRE */} 
        \FOR{$k = 0$ to $K$}
            \STATE Update $\mathbf{H} \leftarrow \hat{\mathbf{A}} \mathbf{H}$ 
            \STATE Update $\mathbf{H}_E \leftarrow [\mathbf{H}_E; \mathbf{H}]$ 
        \ENDFOR
        
       \STATE \textbf{return} $\mathbf{H}_E$ 
    \end{algorithmic}
\end{algorithm}
\subsubsection{Hop-wise Representation Enhancement}

Despite the improvements in feature imputation for attribute-missing nodes achieved by DCFP and HNAI, the features of attribute-complete nodes remain static. Moreover, effective graph clustering often depends on aggregating multi-hop neighborhood information beyond intrinsic features.

To address this, we propose Hop-wise Representation Enhancement (HRE), which captures multi-scale context by progressively performing MP and concatenating features from multiple hops. The final representation is formed by concatenating the features from all hops:
\begin{equation}
\mathbf{H}_E = [\mathbf{H}^{(0)}, \hat{\mathbf{A}}\mathbf{H}^{(0)}, \dots, \hat{\mathbf{A}^{(K)}}\mathbf{H}^{(0)}]^\top \in \mathbb{R}^{n \times (K+1)d}.
\end{equation}
This formulation ensures that each node representation $\mathbf{s}_i \in \mathbb{R}^{(K+1)d}$ encodes hierarchical neighborhood information for downstream clustering.

\subsection{Complexity Analysis}

This section analyzes the computational complexity of the three main components of the proposed DTRGC framework: DCFP, HNAI, and HRE. The total complexity is the sum of the individual components: $\mathcal{O}\Big(K_1 |E| d + \frac{K_1}{T} n c d i + K_2 n c d i + K_2 n d_\text{avg} + K_3 |E| d + K_3 n d \Big)$, where $K_1$ is the number of propagation steps in DCFP, $|E|$ is the number of edges, and $d$ is the feature dimension of the node attribute matrix $\mathbf{H}$, $T$ is the number of iterations for adjusting the adjacency matrix in DCFP, $n$ is the number of nodes, $c$ is the number of clusters, and $i$ is the number of $k$-means iterations, $d_\text{avg}$ is the average degree of nodes, and $K_2$, $K_3$ are the number of iterations for HNAI and HRE, respectively. This total complexity can be simplified to $\mathcal{O}(n + |E|)$, indicating that the proposed DTRGC exhibits a linear relationship with both the number of nodes and the number of edges. 

\section{Experiments}
This section describes the experimental setup and evaluates  DTRGC by answering the following questions:

\noindent\textbf{Q1:} How does the performance of different DGC improve with integrating DTRGC under conditions of missing attributes? 

\noindent\textbf{Q2:} What is the contribution of each component to clustering performance?

\noindent\textbf{Q3:} How does DTRGC perform in comparison to node classification baselines on attribute-missing graphs?

\noindent\textbf{Q4:} What are the time and memory costs of DTRGC?

\noindent\textbf{Q5:} How does the hierarchical processing in DTRGC alleviate issues related to nodes with all unknown neighbors being interconnected?

\subsection{Experimental Settings}

\textbf{Datasets.} To facilitate a comprehensive comparison of the proposed model, six benchmark datasets are selected, including: Cora, Citeseer, Pubmed, Amap, Co.CS, and Co.Physics.

\noindent\textbf{Compared Methods.} To evaluate the imputation effectiveness of DTRGC, several state-of-the-art DGC methods are selected, and their clustering performance is compared in two scenarios: (1) the original clustering performance on attribute-missing graphs; and (2) the clustering performance after DTRGC imputation. This comparison directly assesses the impact of DTRGC on improving clustering performance. The methods included in this comparison are GDCL \cite{zhao2021graph}, AGCN \cite{peng2021attention}, AGC-DRR \cite{gong2022attributed}, HSAN \cite{liuyue_HSAN}, CCGC \cite{yang2023cluster}, and AMGC \cite{tu2024attribute}.

\noindent\textbf{Experiment Setup.} To ensure a fair comparison, 10 experimental iterations are conducted under identical conditions, and the corresponding average results are reported. All experiments run on hardware with a 24GB RTX 3090 GPU and 64GB RAM. The comparison experiments are performed with a 0.6 missing rate, meaning 40\% of node attributes in the original graph are randomly retained, while all attributes of the remaining 60\% of nodes are removed and set as zero vectors. Further details on the experimental setup follow the settings in AMGC \cite{tu2024attribute}. 

\subsection{Clustering Performance Comparison (Q1)}
\label{Clustering Performance Comparison (Q1)}

\begin{table*}[!t]
\caption{The clustering performance on six benchmark datasets (mean \(\pm\) std). The best and runner-up results are highlighted with \textbf{bold} and \underline{underline}, respectively. ``OOM'' means the out-of-memory failure.}
\label{tb:2}
\centering
\resizebox{\textwidth}{!}{

\fontsize{10}{14}\selectfont 
\begin{tabular}{c|c|cccccc|cccccc}
\toprule
\multirow{3}{*}{\textbf{Dataset}} & \multirow{3}{*}{\textbf{Metric}} & \multicolumn{6}{c|}{\textbf{DGC (No Completion)}} & \multicolumn{6}{c}{\textbf{DGC + DTRGC (Pre-Completion)}} \\
\cline{3-14}
 &  & \textbf{GDCL} & \textbf{AGCN} & \textbf{AGC-DRR} & \textbf{HSAN} & \textbf{CCGC} & \textbf{AMGC} & \textbf{GDCL} & \textbf{AGCN} & \textbf{AGC-DRR} & \textbf{HSAN} & \textbf{CCGC} & \textbf{AMGC} \\
 &  & \textbf{IJCAI21} & \textbf{MM21} & \textbf{IJCAI22} & \textbf{AAAI23} & \textbf{AAAI23} & \textbf{AAAI24}& \textbf{IJCAI21} & \textbf{MM21} & \textbf{IJCAI22} & \textbf{AAAI23} & \textbf{AAAI23} & \textbf{AAAI24}\\
\midrule
\multirow{4}{*}{\textbf{Cora}} & ACC & 24.76\(\pm\)1.84 & 39.88\(\pm\)3.18 & 43.26\(\pm\)7.09 & \underline{57.94\(\pm\)1.27} & 37.93\(\pm\)2.68 &  \textbf{66.65\(\pm\)2.04} & 54.23\(\pm\)1.32& 57.81\(\pm\)2.24& 62.27\(\pm\)2.46& \underline{63.81\(\pm\)1.29}& 54.29\(\pm\)1.54& \textbf{70.66\(\pm\)1.25}\\
 & NMI & 4.74\(\pm\)2.21 & 19.16\(\pm\)1.43 & 23.26\(\pm\)8.51 & \underline{42.01\(\pm\)1.24} & 22.37\(\pm\)3.03 & \textbf{47.99\(\pm\)1.67} & 34.21\(\pm\)0.74& 35.59\(\pm\)2.23& 38.24\(\pm\)3.57& \underline{44.12\(\pm\)0.91}& 37.43\(\pm\)0.31& \textbf{56.29\(\pm\)1.14}\\
 & ARI & 0.74\(\pm\)1.04 & 14.09\(\pm\)3.75 & 16.23\(\pm\)8.41 & \underline{33.14\(\pm\)1.78} & 11.45\(\pm\)2.81 & \textbf{43.40\(\pm\)1.90} & 27.41\(\pm\)1.32& 30.21\(\pm\)1.44& 34.54\(\pm\)4.29& \underline{38.75\(\pm\)1.92}& 32.68\(\pm\)1.44& \textbf{46.51\(\pm\)1.46}\\
 & F1 & 7.82\(\pm\)2.16 & 33.96\(\pm\)2.64 & 31.09\(\pm\)8.21 & \underline{58.77\(\pm\)0.89} & 36.91\(\pm\)4.12 & \textbf{61.02\(\pm\)2.38} & 35.46\(\pm\)1.22& 37.21\(\pm\)2.88& 44.51\(\pm\)4.41& \underline{60.24\(\pm\)1.22}& 55.43\(\pm\)0.47& \textbf{66.31\(\pm\)1.11}\\
\midrule
\multirow{4}{*}{\textbf{Citeseer}} & ACC & 27.39\(\pm\)2.19 & 42.22\(\pm\)2.78 & 30.95\(\pm\)4.26 & \underline{44.22\(\pm\)0.52} & 39.63\(\pm\)1.95 & \textbf{60.92\(\pm\)2.01} & 54.26\(\pm\)1.07& 56.13\(\pm\)1.32& 51.17\(\pm\)0.84& \underline{60.34\(\pm\)1.77}& 55.64\(\pm\)2.88& \textbf{64.77\(\pm\)1.31}\\
 & NMI & 7.53\(\pm\)2.38 & 19.01\(\pm\)1.74 & 13.26\(\pm\)3.41 & \underline{22.29\(\pm\)1.39} & 17.13\(\pm\)1.34 & \textbf{32.93\(\pm\)1.71} & 21.45\(\pm\)1.74& 27.17\(\pm\)2.21& 24.97\(\pm\)2.33& \underline{33.27\(\pm\)2.47}& 25.31\(\pm\)0.75& \textbf{37.44\(\pm\)1.24}\\
 & ARI & 2.25\(\pm\)1.08 & 13.31\(\pm\)2.17 & 2.35\(\pm\)3.43 & \underline{13.32\(\pm\)1.08} & 11.83\(\pm\)2.34 & \textbf{33.73\(\pm\)2.25} & 21.45\(\pm\)2.23& 27.88\(\pm\)3.45& 24.31\(\pm\)1.87& \underline{34.45\(\pm\)1.47}& 31.471\(\pm\)1.51& \textbf{38.58\(\pm\)1.56}\\
 & F1 & 12.16\(\pm\)4.02 & 37.22\(\pm\)1.67 & 29.10\(\pm\)5.26 & \underline{42.18\(\pm\)2.54} & 38.88\(\pm\)2.54 & \textbf{57.33\(\pm\)2.05} & 38.48\(\pm\)2.87& 52.91\(\pm\)1.33& 44.31\(\pm\)4.46& \underline{58.91\(\pm\)2.14}& 54.74\(\pm\)2.31& \textbf{60.18\(\pm\)1.32}\\
\midrule
\multirow{4}{*}{\textbf{Amap}} & ACC & 43.34\(\pm\)3.15 & 43.73\(\pm\)2.81 & 61.89\(\pm\)5.96 & \underline{62.36\(\pm\)2.61} & 51.20\(\pm\)2.25 & \textbf{67.56\(\pm\)3.01}& 64.78\(\pm\)2.88& 62.61\(\pm\)3.74& \underline{72.31\(\pm\)2.27}& 71.86\(\pm\)2.34& 67.74\(\pm\)2.41& \textbf{75.65\(\pm\)2.50}\\
 & NMI & 30.94\(\pm\)6.37 & 34.63\(\pm\)4.72 & \underline{55.33\(\pm\)5.32} & 51.42\(\pm\)1.26 & 41.13\(\pm\)2.41 & \textbf{59.81\(\pm\)1.58}
& 59.27\(\pm\)4.46& 60.01\(\pm\)4.79& \underline{63.94\(\pm\)4.37}& 62.57\(\pm\)3.46& 57.61\(\pm\)2.48& \textbf{66.94\(\pm\)2.37}\\
 & ARI & 22.42\(\pm\)4.33 & 22.93\(\pm\)4.97 & \underline{41.43\(\pm\)5.18} & 41.09\(\pm\)2.43 & 30.74\(\pm\)1.06 & \textbf{49.22\(\pm\)2.75}& 48.15\(\pm\)5.78& 45.17\(\pm\)3.13& 52.23\(\pm\)5.33& \underline{54.56\(\pm\)3.15}& 49.79\(\pm\)1.78& \textbf{57.14\(\pm\)3.34}\\
 & F1 & 24.77\(\pm\)3.87 & 32.74\(\pm\)2.51 & \textbf{61.35\(\pm\)7.40} & 55.06\(\pm\)2.40 & 46.17\(\pm\)4.38 & \underline{59.49\(\pm\)3.43}& 58.15\(\pm\)2.78& 64.74\(\pm\)2.45& \underline{67.62\(\pm\)5.13}& 62.48\(\pm\)4.84& 59.41\(\pm\)3.41& \textbf{69.66\(\pm\)2.47}\\
\midrule
\multirow{4}{*}{\textbf{PubMed}} & ACC & 47.00\(\pm\)0.39 & 41.85\(\pm\)2.28 & \underline{60.71\(\pm\)1.72} & \multirow{4}{*}{OOM} & 42.70\(\pm\)1.49 & \textbf{64.56\(\pm\)1.50} & 63.48\(\pm\)0.71& 58.12\(\pm\)2.12& \underline{66.71\(\pm\)1.21}& \multirow{4}{*}{OOM} & 59.15\(\pm\)1.24& \textbf{67.81\(\pm\)2.12}\\
 & NMI & 10.84\(\pm\)0.49 & 0.93\(\pm\)1.08 & \underline{16.26\(\pm\)2.14} & & 2.43\(\pm\)1.75 & \textbf{24.58\(\pm\)2.21} & 18.48\(\pm\)0.24& 12.61\(\pm\)0.74& \underline{27.61\(\pm\)1.87}& & 14.91\(\pm\)1.69& \textbf{30.07\(\pm\)2.13}\\
 & ARI & 7.39\(\pm\)0.27 & 1.37\(\pm\)1.18 & \underline{17.35\(\pm\)2.44} & & 1.24\(\pm\)0.87 & \textbf{24.19\(\pm\)2.39} & 16.57\(\pm\)0.24& 14.71\(\pm\)1.31& \underline{23.74\(\pm\)1.71}& & 16.77\(\pm\)0.47& \textbf{29.24\(\pm\)2.65}\\
 & F1 & 28.36\(\pm\)10.56 & 36.05\(\pm\)0.79 & \underline{59.80\(\pm\)2.43} & & 34.28\(\pm\)3.16 & \textbf{64.52\(\pm\)1.26} & 54.78\(\pm\)4.31& 60.47\(\pm\)0.84& \underline{65.48\(\pm\)2.13}& & 59.41\(\pm\)1.69& \textbf{67.30\(\pm\)1.94}\\
\midrule
\multirow{4}{*}{\textbf{Co.CS}} & ACC & 40.18\(\pm\)1.63 & 52.13\(\pm\)5.88 & 59.08\(\pm\)2.26 & \multirow{4}{*}{OOM} & \underline{64.88\(\pm\)1.51} & \textbf{72.61\(\pm\)1.08} & 59.24\(\pm\)1.44& 61.07\(\pm\)4.23& 64.84\(\pm\)0.78& \multirow{4}{*}{OOM}& \underline{69.48\(\pm\)1.57}& \textbf{74.44\(\pm\)0.94}\\
 & NMI & 41.67\(\pm\)1.29 & 49.51\(\pm\)5.30 & \underline{64.16\(\pm\)1.97} & & 63.21\(\pm\)0.82 & \textbf{73.91\(\pm\)0.39} & 52.67\(\pm\)1.43& 57.81\(\pm\)1.57& 67.75\(\pm\)1.61& & \underline{68.15\(\pm\)0.91}& \textbf{75.27\(\pm\)0.46}\\
 & ARI & 15.32\(\pm\)1.68 & 41.44\(\pm\)7.43 & 49.27\(\pm\)2.86 & & \underline{55.08\(\pm\)2.22} & \textbf{64.62\(\pm\)0.66} & 44.59\(\pm\)2.04& 52.59\(\pm\)3.32& 58.61\(\pm\)2.74& & \underline{64.59\(\pm\)0.77}& \textbf{69.31\(\pm\)0.34}\\
 & F1 & 3.62\(\pm\)2.78 & 25.02\(\pm\)5.54 & 45.09\(\pm\)1.41 & & \underline{54.17\(\pm\)2.76} & \textbf{68.34\(\pm\)1.88} & 44.33\(\pm\)2.41& 54.61\(\pm\)1.67& 59.47\(\pm\)1.81& & \underline{66.87\(\pm\)1.71}& \textbf{71.03\(\pm\)2.12}\\
\midrule
\multirow{4}{*}{\textbf{Co.Physics}} & ACC & \multirow{4}{*}{OOM} & \underline{63.64\(\pm\)4.06} &\multirow{4}{*}{OOM} & \multirow{4}{*}{OOM} & \multirow{4}{*}{OOM} & \textbf{77.68\(\pm\)3.78} & \multirow{4}{*}{OOM}& \underline{66.71\(\pm\)3.94}& \multirow{4}{*}{OOM}& \multirow{4}{*}{OOM}& \multirow{4}{*}{OOM}& \textbf{80.67\(\pm\)1.32}\\
 & NMI & & \underline{34.98\(\pm\)7.68} & & & & \textbf{62.77\(\pm\)1.46} & & \underline{54.31\(\pm\)6.41}& & & & \textbf{65.17\(\pm\)1.04}\\
 & ARI & & \underline{43.94\(\pm\)6.91} & & & & \textbf{69.00\(\pm\)6.25} & & \underline{54.11\(\pm\)4.79}& & & & \textbf{72.61\(\pm\)3.21}\\
 & F1 & & \underline{43.19\(\pm\)6.61} & & & &  \textbf{67.58\(\pm\)2.68} & & \underline{54.52\(\pm\)6.15}& & & &  \textbf{69.31\(\pm\)2.41}\\
\bottomrule

\end{tabular}
}
\end{table*}

The comparison of clustering performance before and after imputation is shown in Table \ref{tb:2}, which presents the experimental results of combining DTRGC with various graph clustering methods across six datasets. The following observations can be made from these results: 

\noindent (1) Across all datasets, incorporating DTRGC consistently leads to significant improvements in clustering performance metrics. For instance, on the Cora dataset, the ACC of the AMGC model increased from $66.65\pm2.04$ to $70.66\pm1.25$, showing an absolute improvement of over 4\%.

\noindent (2) For certain DGC methods that exhibit nearly unacceptable performance under missing attribute scenarios, such as GDCL, incorporating DTRGC significantly boosts their clustering performance to an acceptable level. For example, on the PubMed dataset, the ACC of GDCL improved from $47.00\pm0.39$ to $63.48\pm0.71$. 

\noindent (3) The occurrence of out-of-memory (OOM) failures remains consistent before and after applying DTRGC, indicating that DTRGC does not introduce additional memory overhead that leads to OOM. 

\subsection{T-SNE Visualization (Q1)}

To visually assess the effectiveness of DTRGC, 2D t-distributed Stochastic Neighbor Embedding (t-SNE) \cite{van2008visualizing} is employed to compare the clustering performance of six baseline algorithms before and after imputation. As shown in Fig.\ref{t-sne}, compared to their original versions without DTRGC, the DGC methods with DTRGC imputation exhibit significantly higher intra-cluster compactness and more pronounced inter-cluster separation.

\begin{figure*}[htbp]
\centering
\includegraphics[width=1\linewidth]{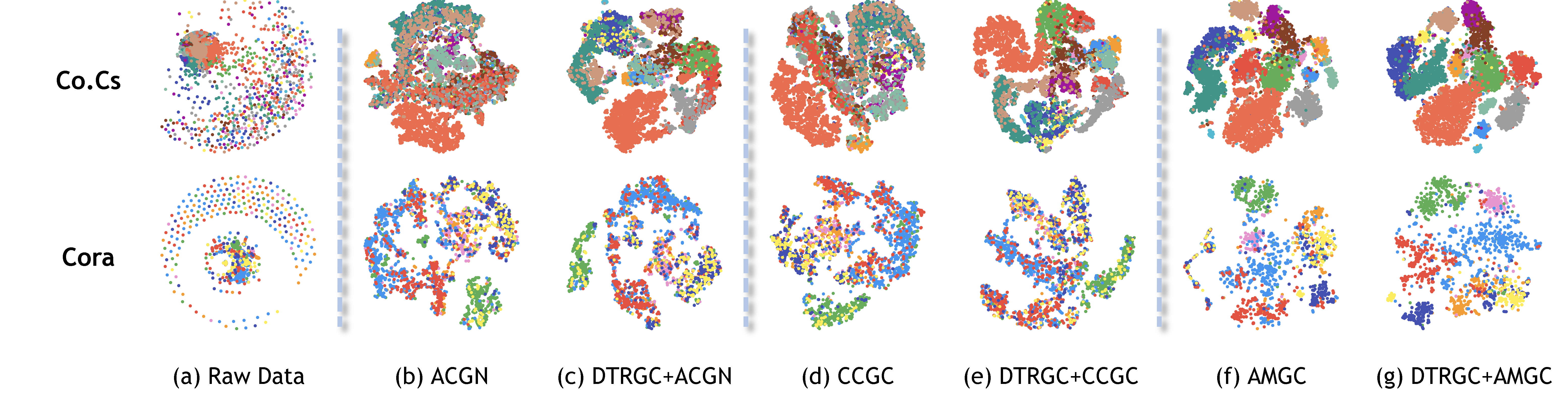}
\caption{T-SNE visualization of node embeddings.}
\label{t-sne}
\end{figure*}

\subsection{Ablation Study (Q2)}
 We progressively remove the three stages of DTRGC to create different variants, aiming to evaluate the independent contribution of each component. It is important to note that DTRGC is not a complete DGC method but rather a plug-and-play imputation module. Therefore, in the ablation experiments, we choose AMGC, which shows the best performance when combined with DTRGC, as the baseline model. 

\begin{table}[!t]
\centering
\caption{Ablation study results on four benchmark datasets, reporting clustering performance metrics (mean $\pm$ std). The best results are highlighted in \textbf{bold}.}
\label{tb:Ablation}

\tabcolsep=1pt 
\renewcommand\arraystretch{1} 
\fontsize{7}{7}\selectfont 
\begin{tabular}{c|c|c|c|c|c}
\toprule
\textbf{Datasets} & \textbf{Method} & \textbf{ACC} & \textbf{NMI} & \textbf{ARI} & \textbf{F1} \\
\midrule
\multirow{8}{*}{Cora} & Base & 66.65$\pm$2.04 & 47.99$\pm$1.67 & 43.40$\pm$1.90 & 61.02$\pm$2.38 \\
 & w/o DCFP & 67.72$\pm$1.75 & 51.71$\pm$1.12 & 44.34$\pm$1.81 & 63.07$\pm$2.07 \\
 & w/o HNAI & 67.84$\pm$1.48 & 49.31$\pm$1.04 & 44.52$\pm$1.61 & 62.91$\pm$1.91 \\
 & w/o HRE & 69.75$\pm$1.32 & 55.48$\pm$1.63 & 45.97$\pm$1.77 & 65.81$\pm$1.41 \\
 & w/o DCFP and HNAI & 67.24$\pm$1.51 & 48.57$\pm$1.24 & 43.84$\pm$1.43 & 62.44$\pm$2.37 \\
 & w/o DCFP and HRE & 67.79$\pm$1.44 & 49.32$\pm$1.33 & 43.91$\pm$1.54 & 62.77$\pm$1.43 \\
 & w/o HNAI and HRE & 67.07$\pm$1.43 & 48.83$\pm$2.34 & 43.84$\pm$1.61 & 62.33$\pm$1.49 \\
 & DTRGC & \textbf{70.66$\pm$1.25} & \textbf{56.29$\pm$1.14} & \textbf{46.51$\pm$1.46} & \textbf{66.31$\pm$1.11} \\
\midrule
\multirow{8}{*}{Citeseer} & Base & 60.92$\pm$2.01 & 32.93$\pm$1.71 & 33.20$\pm$2.25 & 57.33$\pm$2.05 \\
 & w/o DCFP & 62.45$\pm$1.72 & 33.80$\pm$1.58 & 33.84$\pm$1.97 & 58.17$\pm$1.86 \\
 & w/o HNAI & 62.90$\pm$1.53 & 34.75$\pm$1.39 & 35.32$\pm$1.76 & 59.72$\pm$1.67 \\
 & w/o HRE & 63.41$\pm$1.45 & 36.54$\pm$1.47 & 37.13$\pm$1.69 & 59.46$\pm$1.62 \\
 & w/o DCFP and HNAI & 61.98$\pm$1.58 & 33.59$\pm$1.55 & 34.04$\pm$1.54 & 57.92$\pm$1.73 \\
 & w/o DCFP and HRE & 62.95$\pm$1.50 & 34.32$\pm$1.43 & 34.70$\pm$1.61 & 58.50$\pm$1.75 \\
 & w/o HNAI and HRE & 62.42$\pm$1.51 & 33.55$\pm$1.44 & 34.54$\pm$1.59 & 59.12$\pm$1.68 \\
 & DTRGC & \textbf{64.70$\pm$1.31} & \textbf{37.44$\pm$1.24} & \textbf{38.58$\pm$1.56} & \textbf{60.18$\pm$1.32} \\
\midrule
\multirow{8}{*}{Amap} & Base & 67.56$\pm$3.01 & 59.81$\pm$1.58 & 49.22$\pm$2.75 & 59.49$\pm$3.43 \\
 & w/o DCFP & 70.45$\pm$2.55 & 62.20$\pm$1.42 & 52.70$\pm$2.38 & 63.05$\pm$2.95 \\
 & w/o HNAI & 71.32$\pm$2.74 & 62.45$\pm$1.37 & 53.12$\pm$2.51 & 63.58$\pm$3.08 \\
 & w/o HRE & 74.11$\pm$2.62 & 65.15$\pm$1.50 & 55.80$\pm$2.63 & 67.33$\pm$3.22 \\
 & w/o DCFP and HNAI & 69.50$\pm$2.65 & 60.78$\pm$1.49 & 50.32$\pm$2.49 & 61.18$\pm$3.01 \\
 & w/o DCFP and HRE & 70.10$\pm$2.59 & 61.35$\pm$1.43 & 51.50$\pm$2.42 & 62.13$\pm$2.97 \\
 & w/o HNAI and HRE & 71.21$\pm$2.60 & 62.02$\pm$1.48 & 52.60$\pm$2.54 & 62.60$\pm$2.54 \\
 & DTRGC & \textbf{75.65$\pm$2.50} & \textbf{66.94$\pm$2.37} & \textbf{57.14$\pm$3.34} & \textbf{69.66$\pm$2.47} \\
\midrule
\multirow{8}{*}{Pubmed} & Base & 64.56$\pm$1.50 & 24.58$\pm$2.21 & 24.19$\pm$2.39 & 64.52$\pm$1.26 \\
 & w/o DCFP & 66.78$\pm$1.42 & 26.35$\pm$2.01 & 25.20$\pm$2.26 & 65.90$\pm$1.35 \\
 & w/o HNAI & 66.12$\pm$1.34 & 27.12$\pm$1.95 & 25.85$\pm$2.31 & 65.35$\pm$1.41 \\
 & w/o HRE & 66.85$\pm$1.29 & 28.65$\pm$2.03 & 28.14$\pm$2.24 & 66.74$\pm$1.38 \\
 & w/o DCFP and HNAI & 65.20$\pm$1.47 & 25.90$\pm$2.10 & 24.75$\pm$2.28 & 65.11$\pm$1.31 \\
 & w/o DCFP and HRE & 65.65$\pm$1.38 & 26.22$\pm$2.05 & 24.87$\pm$2.30 & 65.75$\pm$1.33 \\
 & w/o HNAI and HRE & 65.44$\pm$1.31 & 26.80$\pm$2.08 & 25.08$\pm$2.27 & 66.10$\pm$1.36 \\
 & DTRGC & \textbf{67.81$\pm$2.12} & \textbf{30.07$\pm$2.13} & \textbf{29.24$\pm$2.65} & \textbf{67.30$\pm$1.94} \\
\bottomrule
\end{tabular}
\end{table}

As summarized in Table \ref{tb:Ablation}, each component of the proposed model significantly contributes to the enhancement of clustering performance. Specifically, taking the accuracy (ACC) as an example, the improvements introduced by DCFP, HNAI, and HRE across the four datasets are as follows: Cora: 0.42\%, 1.14\%, and 0.59\%; Citeseer: 1.50\%, 2.03\%, and 1.06\%; Amap: 3.65\%, 2.54\%, and 1.94\%; Pubmed: 0.88\%, 1.09\%, and 0.64\%, respectively. A consistent pattern of improvement is also observed across other evaluation metrics.

Additionally, when DCFP and HNAI are jointly employed, the resulting performance gains are significantly greater compared to the individual contributions of these two components, yielding improvements of 3.10\%, 4.09\%, 6.55\%, and 3.55\% on the Cora, Citeseer, Amap, and Pubmed datasets, respectively. These findings suggest that: (1) The hierarchical imputation provided by HNAI demonstrates a synergistic effect when integrated with DCFP. (2) The enhancement contributed by HRE appears largely independent of DCFP and HNAI, indicating that its capability to improve representation quality has broader applicability.

\subsection{Node Classification Comparison (Q3)}

As a plug-and-play imputation solution, DTRGC not only performs well in clustering on attribute-missing graphs but is also applicable to a classic learning task on attribute-missing graphs—semi-supervised node classification. The experimental setup and comparison methods follow the current state-of-the-art node classification approach for attribute-missing graphs, PCFI \cite{um2023confidence}. Specifically, under a missing rate of 0.995, we compare the node classification performance of DTRGC with GCNMF \cite{taguchi2021graph}, PaGNN \cite{jiang2020incomplete}, and FP \cite{rossi2022unreasonable}. For plug-and-play imputation methods, we use vanilla GCN for the subsequent node classification tasks. 

Table \ref{tb:Node Classification Performance} presents the experimental results under a 0.995 missing rate scenario. It is evident that both DTRGC and PCFI outperform other methods across all datasets. Notably, DTRGC surpasses PCFI on the Cora, Computers, and ogbn-Arxiv datasets, achieving an accuracy of 76.12\% on Cora, 79.43\% on Computers, and 69.24\% on ogbn-Arxiv. Despite being initially designed for clustering, DTRGC demonstrates competitive performance in semi-supervised node classification, highlighting its robustness and strong generalization capability across a variety of graph learning scenarios.

\subsection{Time and Memory Consumption (Q4)}
This section is used to answer the time and memory costs of DTRGC. Since DTRGC is a preprocessing step, we do not compare it with other methods but instead report its time and memory usage. As shown in Table \ref{yuchuli-time}, the preprocessing time and memory consumption of DTRGC exhibit a linear relationship across different datasets. These overheads remain within an acceptable range for all datasets. Moreover, since DTRGC serves as a one-time preprocessing step, its computational cost is negligible in the long run.

\begin{table}[ht]
\centering
\caption{Time and memory cost on different datasets.}
\label{yuchuli-time}
\renewcommand\arraystretch{1}
\fontsize{7}{7}\selectfont
\begin{tabular}{c|c|c}  
\toprule
\textbf{Dataset}       & \textbf{Time Cost (s)} & \textbf{CPU Memory Cost (MB)} \\
\midrule
Cora         & 31.34                  & 43.21                         \\
CiteSeer     & 30.99                  & 42.90                         \\
Amap         & 488.71                 & 673.00                        \\
PubMed       & 246.44                 & 339.92                        \\
Co.CS        & 386.01                 & 532.13                        \\
Co.Physics   & 1089.34                & 1502.43                       \\
\bottomrule
\end{tabular}
\end{table}

\subsection{Effectiveness of Hierarchical Processing (Q5)}

To validate the effectiveness of the HNAI module within DTRGC, a set of experiments is designed to track the evolution of three types of nodes across different datasets. This aims to analyze their transformation over multiple iterations. As shown in Fig. \ref{xiajiang}, the trend in the number of three types of nodes is analyzed over multiple iterations across six datasets: Cora, PubMed, CiteSeer, CS, Amap, and Physics. The experimental results indicate that the HNAI module within DTRGC significantly reduces the number of ``nodes with partially known neighbors'' and ``nodes with all unknown neighbors'', especially in the initial iterations, where the counts of these nodes drop dramatically. As the iterations progress, subgraphs made up of interconnected ``nodes with all unknown neighbors'' gradually gain reliable information from the surrounding nodes, reducing the negative impact of indiscriminate propagation on clustering.

\begin{figure}[htbp]
\centering
\includegraphics[width=1\linewidth]{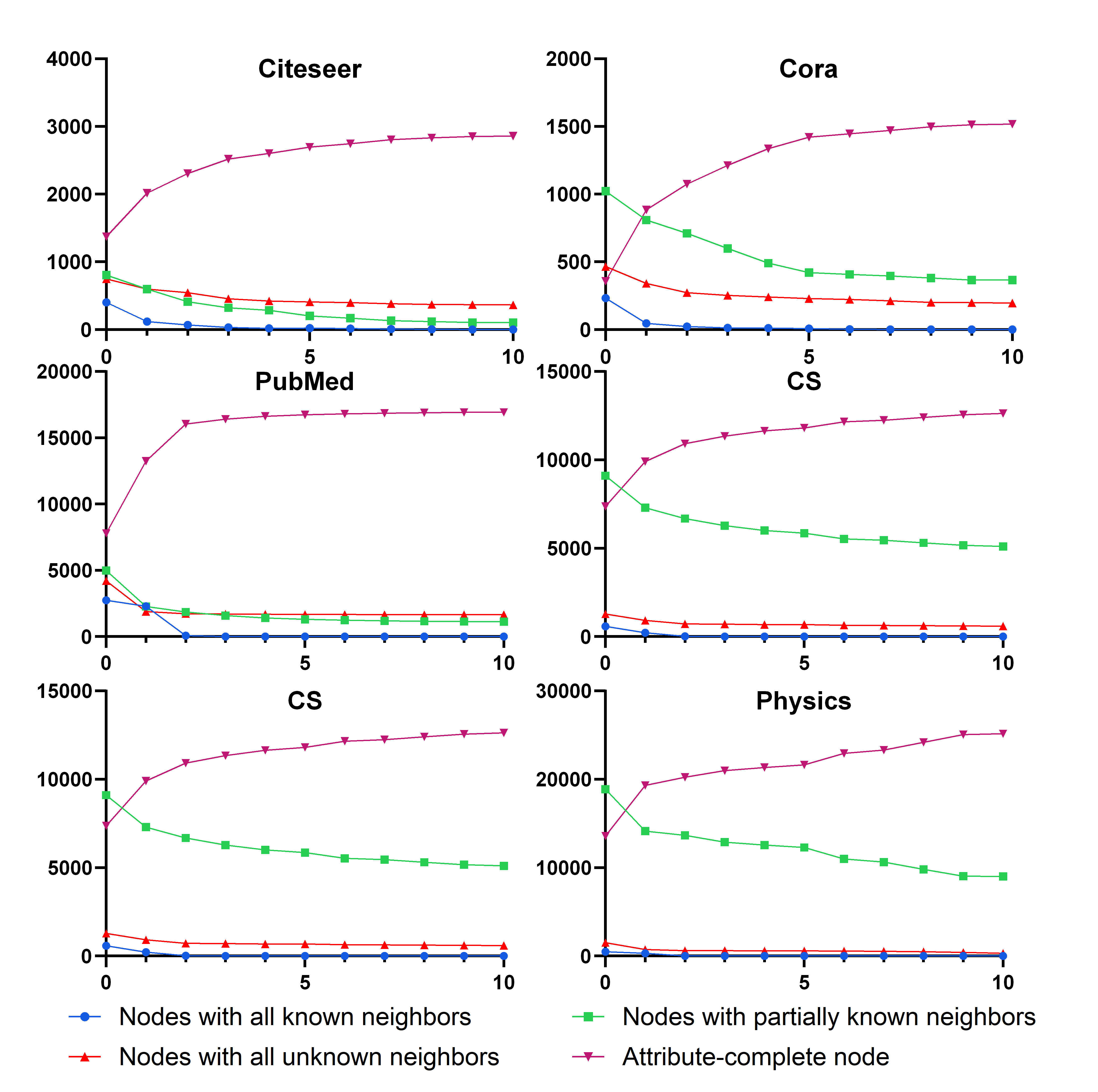}
\caption{The number of nodes for each of the three types over different iterations.}
\label{xiajiang}
\end{figure}

\section{Conclusion and Future Work}

\begin{table}[!t]
\caption{Comparison of classification accuracy across benchmark datasets for different imputation methods in semi-supervised node classification tasks with a missing rate of 0.995 (mean \(\pm\) std). The best results are highlighted in \textbf{bold}.}
\centering
\label{tb:Node Classification Performance}
\resizebox{\columnwidth}{!}{
\fontsize{8}{8}\selectfont 
\tabcolsep=2pt
\renewcommand\arraystretch{1}
\begin{tabular}{c|c|c|c|c|c}  
\toprule
Dataset    & GCNMF            & PaGNN            & FP               & PCFI                                       & DTRGC                                      \\
\midrule
Cora       & 29.20\(\pm\)0.71 & 30.55\(\pm\)8.85 & 72.84\(\pm\)2.85 & 75.49\(\pm\)2.10                           & \textbf{76.12\(\pm\)2.03} \\
Citeseer   & 24.50\(\pm\)1.52 & 25.69\(\pm\)3.98 & 59.76\(\pm\)2.47 & \textbf{66.18\(\pm\)2.75}                  & 65.34\(\pm\)1.64                           \\
PubMed     & 40.19\(\pm\)0.95 & 50.82\(\pm\)4.61 & 72.69\(\pm\)2.66 & \textbf{74.66\(\pm\)2.26}                  & 73.31\(\pm\)2.47                           \\
Photo      & 26.82\(\pm\)6.33 & 66.91\(\pm\)3.99 & 86.57\(\pm\)1.50 & \textbf{87.70\(\pm\)1.29}                  & 86.94\(\pm\)1.13                           \\
Computers  & 30.59\(\pm\)9.81 & 56.50\(\pm\)3.29 & 77.45\(\pm\)1.59 & 79.25\(\pm\)1.19                           & \textbf{79.43\(\pm\)1.45} \\
ogbn-Arxiv & OOM              & OOM              & 68.23\(\pm\)0.27 & 68.72\(\pm\)0.28                           & \textbf{69.24\(\pm\)0.47} \\
\bottomrule
\end{tabular}
}
\end{table}

This paper proposes a framework, termed DTRGC, to address challenges in clustering attribute-missing graphs. Unlike existing imputation methods that often overlook varying levels of neighborhood information availability, DTRGC prioritizes nodes with richer neighborhood information, allowing them to serve as reliable anchors for subsequent imputation of more challenging nodes. Experimental results on six widely used graph datasets confirm that DTRGC significantly enhances the clustering performance of attribute-missing graphs across various DGC methods. As a direction for future research, we intend to develop a parametric extension of DTRGC, aiming to further improve its adaptability, model complexity, and performance across a broader range of heterogeneous graph datasets and real-world applications.

\section*{Acknowledgments}
This work is supported by the Key Project of the Xiangjiang Laboratory (Nos. 23XJ02003), the National Natural Science Foundation of China (Nos. 62441618, 62325604, and 62276271), and the Natural Science Foundation of Hainan University (Nos. XJ2400009401). 

\bibliographystyle{ACM-Reference-Format}
\balance
\bibliography{sample-base}

\end{document}